\documentclass[letterpaper, journal, 10pt]{IEEEtran}
\usepackage{amsmath, amsfonts, amsthm}
\usepackage{algorithm}
\usepackage{algpseudocodex}
\usepackage{graphicx}
\usepackage{bm}
\usepackage{xcolor}
\usepackage{enumitem}
\usepackage{hyperref}
\usepackage{tabularray}
\usepackage{orcidlink}
\usepackage[style=ieee, citestyle=numeric-comp, url=false, natbib=true, isbn=false, doi=false]{biblatex}
\usepackage{stfloats}
\usepackage[spacing]{microtype}
\DeclareSourcemap{
    \maps[datatype=bibtex, overwrite]{
        \map{\step[fieldset=editor, null]}
    }
}

\algrenewcommand\algorithmicindent{1em}

\bibliography{references}
\UseTblrLibrary{booktabs}

\newcommand{\email}[1]{\href{mailto:#1}{#1}}
\renewcommand{\Vec}[1]{\bm{#1}}
\newcommand{\Mat}[1]{\bm{#1}}
\newcommand{\inv}{^{-1}}
\newcommand{\T}{^\top}

\newcommand{\SFN}[1]{{\lVert{#1}\rVert_\mathrm{F}^2}}
\newcommand{\dataset}[1]{\mathcal{#1}}
\newcommand{\argmin}[1]{\underset{#1}{\operatorname{argmin}}~}

\newcommand{\ReLU}{\operatorname{ReLU}}

\newcommand{\onehot}{\operatorname{onehot}}
\newcommand{\X}{\Mat{X}}
\newcommand{\Y}{\Mat{Y}}
\newcommand{\R}{\Mat{R}}
\newcommand{\Q}{\Mat{Q}}
\newcommand{\W}{\Mat{W}}
\newcommand{\Wh}{\hat{\Mat{W}}}
\newcommand{\I}{\Mat{I}}
\newcommand{\K}{\Mat{K}}
\newcommand{\Ximg}{\Mat{\mathcal{X}}}
\newcommand{\f}{\Vec{f}}

\newcommand{\Wbackbone}{\Mat{W}_{\text{backbone}}}
\newcommand{\Wrand}{\Mat{W}_{\text{buffer}}}
\newcommand{\MU}{\boldsymbol{\mu}}
\newcommand{\NU}{\boldsymbol{\nu}}
\newcommand{\SIGMA}{\boldsymbol{\sigma}}

\newtheorem{theorem}{Theorem}

\begin{document}

\title{Online Analytic Exemplar-Free Continual Learning with Large Models for Imbalanced Autonomous Driving Task}

\author{
	\thanks{Copyright~\copyright~2024 IEEE. Personal use of this material is permitted. However, permission to use this material for any other purposes must be obtained from the IEEE by sending a request to \email{pubs-permissions@ieee.org}.}

    Huiping Zhuang\orcidlink{0000-0002-4612-5445}\thanks{
        Huiping Zhuang (e-mail: \email{hpzhuang@scut.edu.cn}),
        Kai Tong (e-mail: \email{wikaitong@mail.scut.edu.cn}), and
        Ziqian Zeng (e-mail: \email{zqzeng@scut.edu.cn})
        are with the Shien-Ming Wu School of Intelligent Engineering, South China University of Technology, Guangdong 510641, China.
    },
    \and
    Di Fang\orcidlink{0009-0004-8135-2354}\thanks{
        Di Fang (e-mail: \email{fti@mail.scut.edu.cn}) and Cen Chen (e-mail: \email{chencen@scut.edu.cn})
        are with the School of Future Technology, South China University of Technology, Guangdong 510641, China.
        Cen Chen is also with the Pazhou Laboratory, Guangzhou 510330, China.
    },
    \and
    Kai Tong\orcidlink{0009-0001-6073-8918},
    \and
    Yuchen Liu\orcidlink{0009-0001-3831-1168}\thanks{
        Yuchen Liu (e-mail: \email{liuyuchen@connect.hku.hk}) is with the Department of Mechanical Engineering, the University of Hong Kong, Hong Kong 999077, China.
    },\\
    \and
    Ziqian Zeng\orcidlink{0000-0003-0060-7956}$^{*}$,
    \and
    Xu Zhou\orcidlink{0000-0002-0764-0620}\thanks{Xu Zhou (e-mail: \email{zhxu@hnu.edu.cn}) is with the Department of Information Science and Engineering, Hunan University, Hunan 410082, China.},
    \and
    Cen Chen\orcidlink{0000-0003-1389-0148},~\IEEEmembership{Senior Member,~IEEE}
    \thanks{$^{*}$Corresponding author: Ziqian Zeng.}
}

\maketitle
\vspace{-3em}

\begin{abstract}
    In autonomous driving, even a meticulously trained model can encounter failures when facing unfamiliar scenarios. One of these scenarios can be formulated as an online continual learning (OCL) problem. That is, data come in an online fashion, and models are updated according to these streaming data. Two major OCL challenges are catastrophic forgetting and data imbalance. To address these challenges, we propose an Analytic Exemplar-Free Online Continual Learning algorithm (AEF-OCL). The AEF-OCL leverages analytic continual learning principles and employs ridge regression as a classifier for features extracted by a large backbone network. It solves the OCL problem by recursively calculating the analytical solution, ensuring an equalization between the continual learning and its joint-learning counterpart, and works without the need to save any used samples (i.e., exemplar-free). Additionally, we introduce a Pseudo-Features Generator (PFG) module that recursively estimates the mean and the variance of real features for each class. It over-samples offset pseudo-features from the same normal distribution as the real features, thereby addressing the data imbalance issue. Experimental results demonstrate that despite being an exemplar-free strategy, our method outperforms various methods on the autonomous driving SODA10M dataset. Source code is available at \url{https://github.com/ZHUANGHP/Analytic-continual-learning}.
\end{abstract}


\section{Introduction}
    \IEEEPARstart{A}{utonomous} driving technology \cite{AutonomousDriving2023IntelligentVehicles,Humanlikedriving2018TransactionsonVehicularTechnology, ASurveyofAutonomousDriving2020IEEEAccess, MTYGNN_Xiaofeng_TITS2022} is currently grappling with the complex and diverse challenges presented by real-world scenarios. These scenarios are marked by a wide range of factors, including varying weather conditions like heavy snowfall, as well as different road environments \cite{MGSTC_Chen_AAAI2019, MGSTC_Chen_TKDD2020}. Even well-trained autonomous driving models often struggle to navigate through these unfamiliar circumstances.

    The advent of large-scale models \cite{BERT2019NAACL}, characterized by their extensive parameter counts and training data, has led to substantial improvements in the feature extraction capabilities of these models. This increase in the parameter number has enabled the utilization of various downstream applications, offering enhanced feature extraction capabilities crucial for high-accuracy tasks such as classification, segmentation, and detection to aid autonomous driving. However, despite these advancements, the goal of achieving efficient and dynamic learning in complex autonomous driving environments and scenes remains unachieved.

    One of these efficient and dynamic learning challenges encountered in autonomous driving can be formulated as a continual learning (CL) problem \cite{li2018LWF,rebuffi2017icarl} in an online setting. That is, models are updated according to these streaming data in an online fashion. However, this inevitably leads to the so-called \textit{catastrophic forgetting} \cite{CF_Bower_PLM1989, CF_Ratcliff_PR1990}, where models lose grip of previously learned knowledge when obtaining new information. Furthermore, the online data streaming manner often accompanies a \textit{data imbalance} issue \cite{OA3_Zhang_KDD2018}, with information in different categories containing varying data counts in general. For instance, in the autonomous driving dataset SODA10M \cite{han2021soda10m}, the \textit{Tricycle} category contains just 0.3\% of the training set, whereas the \textit{Car} category accounts for 55\%. This imbalance issue exacerbates the forgetting problem, rendering more difficult learning of continuous knowledge.

    To address the above streaming task, the online continual learning (OCL) has been introduced. OCL methods belong to the CL category with an online constraint, striving to preserve old knowledge while learning new information from streaming data. The OCL is more challenging as the streaming data can only be updated once (i.e., one epoch). Like the CL, existing OCL methods can be roughly categorized into two groups, namely replay-based and exemplar-free methods. The replay-based OCL keeps a small subset of trained samples and reduces catastrophic forgetting by mixing them during the following training tasks. Replay-based methods usually obtain good performance but invade data privacy by keeping samples.

    The exemplar-free OCL, on the other hand, tries to avoid catastrophic forgetting while adhering to an additional exemplar-free constraint. That is, no trained samples are stored for the following training tasks. This category of OCL is more challenging but has attracted increasing attention. Among the real-world autonomous driving scenarios, exemplar-free OCL methods are often needed, driven by concerns related to online sample flow, data privacy, and algorithmic simplicity. However, the performance of existing exemplar-free methods remains inadequate, especially in the online streaming setting.

    To tackle the catastrophic forgetting problem and the data imbalance issue, in this paper, we propose an Analytic Exemplar-Free Online Continual Learning algorithm (AEF-OCL). The AEF-OCL adopts an analytic learning approach \cite{brmp2021}, which replaces the back-propagation with a recursive least-squares (RLS) like technique. In traditional scenarios, the combination of RLS and OCL has demonstrated promising primary results \cite{zhuang2022acil,zhuang2023gkeal}. The contributions of our work are summarized as follows:
    \begin{itemize}[leftmargin=1em]
        \item We introduce the AEF-OCL, a method for OCL that eliminates the need for exemplars. The AEF-OCL offers a recursive analytical solution for OCL and establishes an equivalence to its joint-learning counterpart, ensuring that the model firmly retains previously learned knowledge. This approach effectively addresses the issue of catastrophic forgetting without storing any past samples.
        \item We introduce a Pseudo-Features Generator (PFG) module. This module conducts a recursive calculation of task-specific data distribution and generates pseudo-data by considering the distribution of the current task's feature to tackle the challenge of data imbalance.
        \item Theoretically, we demonstrate that the AEF-OCL achieves an equivalence between the CL structure and its joint-learning counterpart by adopting all the data.
        \item We apply the AEF-OCL by adopting a large-scale pre-trained model to address the CL tasks in autonomous driving. Our experiments on the SODA10M dataset \cite{han2021soda10m} demonstrate that the AEF-OCL performs well in addressing OCL challenges within the context of autonomous driving.
    \end{itemize}

\section{Related Works}\label{sec:related_works}
    In this section, we first review the details of the autonomous driving dataset SODA10M and its metric. Subsequently, we survey commonly seen CL methods, including replay-based and exemplar-free ones. Then, we summarize the OCL methods, which are mainly replay-based approaches. Finally, we review CL methods designed for the data imbalance issue.

    \subsection{The SODA10M dataset}
        In light of the popularity of autonomous driving technology, datasets pertinent to this field have obtained significant attention. As a notable dataset in this area, the SODA10M dataset \cite{han2021soda10m} comprises 10 million unlabeled images and 20,000 labeled images captured from vehicular footage across four cities. In this study, we restrict our focus to the labeled images to examine OCL tasks. Building upon the SODA10M labeled images, the CLAD \cite{verwimp2023clad} introduces a CL benchmark for autonomous driving. This approach partitions the labeled images of the SODA10M dataset into six tasks, distributed over three days and three nights based on the capture time. Models are trained sequentially on these six tasks, with verification conducted after each task.

    \subsection{Continual Learning Methods}
        In the realm of CL methods, we can broadly classify them into two distinct categories: replay-based and exemplar-free strategies. The former, replay-based techniques, utilize stored historical samples throughout the training process as a countermeasure to the catastrophic forgetting issue, thereby enhancing the overall performance. On the other hand, the exemplar-free methods aim to comply with an additional constraint that avoids the retention of trained samples for subsequent training stages. This type of OCL presents a greater challenge, yet it has been garnering increasing interest.

    \subsubsection{Replay-based CL}
        The paradigm of replay-based CL, which enhances the model's capacity to retain historical knowledge through the replay of past samples, has been increasingly recognized for its potential to mitigate the issue of catastrophic forgetting. The pioneering work by the iCaRL \cite{rebuffi2017icarl} marks the inception of this approach, leading to the subsequent development of numerous methods due to its substantial performance improvements. \citet{EEIL_2018_ECCV} propose a novel approach that incorporates a cross-distillation loss achieved via a replay mechanism that combines two loss functions: cross-entropy loss for learning new classes and distillation loss to preserve previously acquired knowledge of old classes.
        In a deviation from the conventional softmax layer, the LUCIR \cite{LUCIR_Hou_CVPR2019} introduces a cosine-based layer. The PODNet \cite{douillard2020podnet} implements an efficient space-based distillation loss to counteract forgetting, with a particular focus on significant transformations, which has yielded encouraging results.
        The FOSTER \cite{FOSTER2022ECCV} employs a two-stage learning paradigm that initially expands the network size, and subsequently reduces it to its original dimensions. The AANets \cite{AANet_2021_CVPR} incorporates a stable block and a plastic block to strike a balance between stability and plasticity. In general, replay-based CL achieves adequate results, but due to issues of data privacy and training costs, it is not very suitable for practical applications.

    \subsubsection{Exemplar-free CL}
        Exemplar-free CL methods do not require storing historical samples, making them more suitable for privacy-focused applications like autonomous driving. Exemplar-free CL can be roughly categorized into three branches: regularization-based CL, prototype-based CL, and the recently emerged analytic CL (ACL).

        \par \textit{Regularization-based CL} creates an innovative loss function to encourage the model to re-engage with previously acquired knowledge to prevent the model from forgetting. Methods such as the less-forgetting learning \cite{LessForgetting_2016_arXiv} and the LwF \cite{li2018LWF} introduce knowledge distillation \cite{KD_Hinton_arXiv2015} into their loss function to prevent catastrophic forgetting caused by activation drift. To prevent the drift of the important weights, the EWC \cite{EWC2017nas} introduces regularization to the network parameters, employing a diagonal approximation of the Fisher information matrix to encapsulate the a priori importance, and the R-EWC \cite{liu2018rn} endeavors to discover a more appropriate alternative to the Fisher information matrix. However, when the number of tasks is large, especially in OCL scenarios, regularization-based methods still face a serious catastrophic forgetting problem.

        \par \textit{Prototype-based CL} has emerged as a viable solution to the catastrophic forgetting problem by maintaining prototypes for each category, thereby ensuring new and old categories do not share overlapping representations. For instance, the PASS \cite{Zhu_2021_CVPR} differentiates prior categories through the augmentation of feature prototypes. In a similar vein, the SSRE \cite{Zhu_2022_CVPR} introduces a prototype selection mechanism that incorporates new samples into the distillation process, thereby emphasizing the dissimilarity between the old and new categories. The ProCA \cite{ProCA_Lin_ECCV2022} adapts the source model to a class-incremental unlabeled target domain. Furthermore, the FeTrIL \cite{Petit_2023_WACV} offers another innovative solution to mitigate forgetting. It generates pseudo-features for old categories, leveraging new representations. However, a major challenge to the prototype-based CL is that old prototypes may be inaccurate during the CL process. Several approaches \cite{PRAKA_Shi_ICCV2023, ESSA_Cheng_TCSVT2024, NAPA-VQ_Tamasha_ICCV2023} are proposed to address this issue.

        \par \textit{ACL} is a recently developed exemplar-free approach inspired by pseudoinverse learning \cite{GUO2004101}. In ACL, classifiers are trained using the RLS-like technique to generate a closed-form solution to overcome the inherent drawbacks associated with back-propagation, such as the gradient vanishing/exploding, divergence during iterative processes, and long training epochs. The ACIL \cite{zhuang2022acil} restructures CL programs into a recursive analytic learning process, eliminating the necessity of storing samples through the preservation of the correlation matrix. The GKEAL \cite{zhuang2023gkeal} focuses on few-shot CL scenarios by leveraging a Gaussian kernel process that excels in zero-shot learning. The RanPAC \cite{RanPAC_McDonnell_NeurIPS2023} just simply replaces the recursive classifier of the ACIL with an iterative one. To enhance the ability of the classifier, the DS-AL \cite{Zhuang_DSAL_AAAI2024} introduces another recursive classifier to learn the residue, and the REAL \cite{REAL_He_arXiv2024} introduces the representation enhancing distillation to boost the plasticity of backbone networks. The AFL \cite{AFL_Zhuang_arXiv2024} extends the ACL to federated learning, transitioning from temporal increment to spatial increment, and \citet{LSSE_Liu_ICLR2024} apply ACL to reinforcement learning. The ACL is an emerging CL branch, exhibiting strong performance due to its equivalence between CL and joint-learning, in which all the data are adopted altogether to train the model. Our AEF-OCL belongs to ACL. Compared with the latest work, a PFG module is applied to solve the data imbalance problem. Our AEF-OCL incorporates ACL methods into OCL and achieve state-of-the-art results.

    \subsection{Online Continual Learning}
        The OCL task aims to acquire knowledge of new tasks from a data stream, with each sample being observed only once. A prominent solution to this task is provided by ER \cite{hayes2019ER}. It employs a strategy of storing samples from previous tasks and then randomly selects a subset of these samples as exemplars merged with new samples during the training of subsequent tasks. To select valuable samples from the memory, memory retrieval strategies such as the MIR \cite{aljundi2019MIR} and the ASER \cite{shim2021aser} are utilized. The SCR \cite{SCR_2021_CVPR} gathers samples from the same category closely together in the embedding space, while simultaneously distancing samples from dissimilar categories during replay-based training. The PCR \cite{PCR_2023_CVPR} couples the proxy-based and contrastive-based replay manners, and replaces the contrastive samples of anchors with corresponding proxies. \citet{OHO_Liu_AAAI2023} formulate the hyper-parameter optimization as an online Markov Decision Process. Imbalanced data in the transportation will exacerbate the problem of catastrophic forgetting in existing exemplar-free OCL methods.

    \subsection{CL with Large Pre-trained Models}
        \par Large pre-trained models bring backbone networks with strong feature representation ability to the CL. On the one hand, inspired by fine-tuning techniques in NLP \cite{P-Tuning_Lester_ACL2021, LoRA_Hu_ICLR2022,Dap-SiMT_Zhao_IJMLC2024}, the DualPrompt \cite{DualPrompt_Wang_ECCV2022}, the CODA-Prompt \cite{CODA-Prompt_Smith_CVPR2023}, and the MVP \cite{MVP-GCIL_Moon_ICCV2023} introduce prompts into CL, while the EASE \cite{EASE_Zhou_CVPR2024} introduces a distinct lightweight adapter for each new task, aiming to create task-specific subspace. On the other hand, the SimpleCIL \cite{SimpleCIL_Zhou_IJCV2024} shows that with the help of a simple incremental classifier and a frozen large pre-trained model as a feature extractor that can bring generalizable and transferable feature embeddings, it can surpass many previous CL methods. Thus, it is with great potential to combine the large pre-trained models with the CL approaches with a powerful incremental classifier, such as the SLDA \cite{SLDA_Hayes_CVPR2020} and the ACL methods \cite{zhuang2022acil, zhuang2023gkeal, RanPAC_McDonnell_NeurIPS2023, Zhuang_DSAL_AAAI2024}.

    \subsection{Data Imbalanced Continual Learning}
        \par The data imbalance issue is one of the most significant challenges in CL for autonomous driving. This imbalance can lead to models overlooking categories with fewer training samples and exacerbating the catastrophic forgetting issue. Several methods are proposed to address this, including the LUCIR \cite{LUCIR_Hou_CVPR2019}, the BiC \cite{BiC_Wu_CVPR2019}, PRS \cite{PRS_Kim_ECCV2020}, and the CImbL \cite{CImbL_He_CVPR2021}. They focus more on the imbalance issue in class incremental learning. The LST \cite{LST_Hu_CVPR2020} and the ActiveCIL \cite{ActiveCIL_Belouadah_ECCV2020} are designed for few-shot CL and active CL, respectively. \citet{LTCIL_Liu_ECCV2022} propose a two-stage learning paradigm, bridging the existing CL methods to imbalanced CL. The experiments conducted by them on long-tailed datasets inspire a series of subsequent works \cite{DRC_Chen_ICCV2023, CLAD_Xu_AAAI2024, DGR_He_CVPR2024, ISPC_Wang_CVPR2024, DAP_Hong_IJCAI2024, JIOC_Wang_IJCAI2024, APART_Qi_ML2024}. In OCL, the CBRS \cite{CBRS_Chrysakis_ICML2020} introduces a memory population approach for data balance, the CBA \cite{CBA_Wang_ICCV2023} proposes an online bias adapter, the LAS \cite{LAS_Huang_TMLR2024} introduces a logit adjust softmax to reduce inter-class imbalance, and the DELTA \cite{DELTA_Raghavan_CVPR2024} introduces a decoupled learning approach to enhance learning representations and address the substantial imbalance.

\section{Proposed Method}\label{sec:proposed_method}

    \begin{figure*}[!t]
        \centering
        \includegraphics[width=\linewidth]{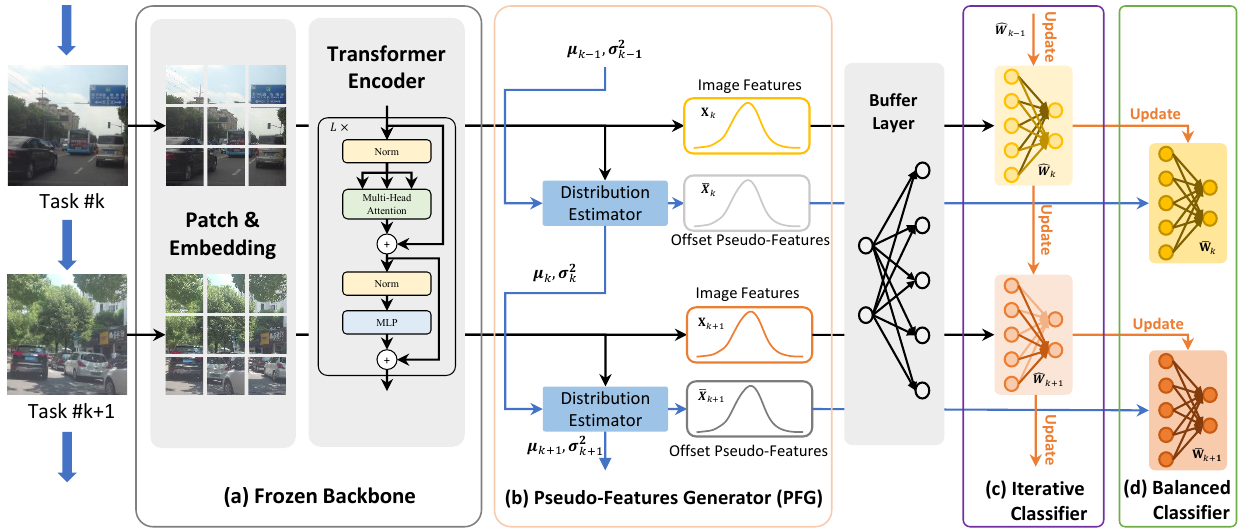}
        \caption{
            The training process of our proposed method includes:
            (\textbf{a}) a large universal frozen pre-trained backbone such as a ViT without its classification head;
            (\textbf{b}) a pseudo-features generator that estimates the mean and the variance of features recursively and generates the offset pseudo-features in an estimated normal distribution to balance the training data;
            (\textbf{c}) an iterative ridge regression classifier that iteratively updates its weight with real features only;
            (\textbf{d}) a balanced ridge regression classifier for inference that updates its weight from the iterative classifier using offset pseudo-features generated at each task.
        }\label{fig:flowchart}
    \end{figure*}

    \subsection{Overview}
        The AEF-OCL has 4 steps. Firstly, a frozen backbone is used to extract the features of the images. Secondly, we introduce a PFG module to solve the challenge of data imbalance. A frozen random initialized linear buffer layer is adopted to project the feature space into a higher one, making the feature suitable for ridge regression \cite{hoerl1970ridge}. Finally, we replace the original classification head of the model with a ridge regression classifier. As shown in Fig. \ref{fig:flowchart}, we train the ridge regression classifier recursively to classify the features obtained by the frozen backbone and the random buffer layer.

        To solve the problem caused by imbalanced data, the PFG module generates pseudo-features with their corresponding labels of each minor class (i.e., classes with less number of samples) to compensate for the imbalanced training samples. We assume that the feature distribution is normal. Hence, we estimate the mean and variance recursively and generate the offset pseudo-features in the same normal distribution as real features to balance the training dataset.

        The pseudo-features generated by the distribution estimator will subsequently entered into the same training process as those real samples. Notably, these generated features only influence the current classifier for inference, without updating the iterative classifiers. Thus, we can have a balanced classifier for the inference procedure. The pseudo-code of the overall training process is listed in Algorithm \ref{algo:overall}.

        \begin{algorithm}
            \caption{The training process of the AEF-OCL}\label{algo:overall}
            \begin{algorithmic}
                \Procedure{trainForOneBatch}{$\dataset{D}_k$}
                    \LComment{The $k$-th sample in the dataset $\dataset{D}_k$ is $(\Ximg, y)$.}
                    \ForAll {$(\Ximg, y, i) \in \dataset{D}_k$}
                        \LComment{Feature extraction}
                        \State $\f_i \gets f(\Ximg, \Wbackbone)$
                        \State $\Vec{x}_i \gets \ReLU\left(\f_i\Wrand\right)$
                        \State $\Vec{y}_i \gets \onehot(y)$

                        \LComment{Update statistics}
                        \State $n_{y} \gets n_{y} + 1$
                        \State $\MU^{(y)} \gets \frac{1}{n_y}\f_i + \frac{n_y - 1}{n_y}\MU^{(y)}$
                        \State ${\NU^{(y)}} \gets \frac{1}{n_y}\f_i^2 + \frac{n_y - 1}{n_y}{\NU^{(y)}}$
                        \State ${\SIGMA^{(y)}} \gets \sqrt{\frac{n_y}{n_y - 1}(\NU^{(y)} - {\MU^{(y)}}^2)}$
                    \EndFor
                    \State $\X_{k} \gets \begin{bmatrix}\Vec{x}_1\T & \Vec{x}_2\T & \cdots\end{bmatrix}\T$
                    \State $\Y_{k} \gets \begin{bmatrix}\Vec{y}_1\T & \Vec{y}_2\T & \cdots\end{bmatrix}\T$

                    \State \LComment{Train the iterative classifier}
                    \State $\Wh_k, \R_k \gets$ \Call{Update}{$\Wh_{k-1}$, $\R_{k-1}$, $\X_{k}$, $\Y_{k}$}

                    \State \LComment{Generate pseudo-features}
                    \State $n_{\text{max}} \gets \max\{n_0, n_1, \cdots, n_{C-1}\}$
                    \For{$c \gets 0$ \textbf{to} $C - 1$}
                        \For{$i \gets 1$ \textbf{to} $n_{\text{max}} - n_c$}
                            \State Sample $\overline{\f}_i$ from $\mathcal{N}(\MU^{(c)}, {\SIGMA^{(c)}}^2)$
                            \State $\overline{\Vec{x}}_i \gets \ReLU\left(\bar{\f_i}\Wrand\right)$
                            \State $\overline{\Vec{y}}_i \gets \onehot(c)$
                        \EndFor
                        \State $\overline{\X}_{k,c} \gets \begin{bmatrix}\Vec{\overline{x}}_1\T & \Vec{\overline{x}}_2\T & \cdots\end{bmatrix}\T$
                        \State $\overline{\Y}_{k,c} \gets \begin{bmatrix}\Vec{\overline{y}}_1\T & \Vec{\overline{y}}_2\T & \cdots\end{bmatrix}\T$
                    \EndFor
                    \State $\overline{\X}_k \gets \begin{bmatrix}\overline{\X}_{k,0}\T & \overline{\X}_{k,1}\T & \cdots & \overline{\X}_{k,C-1}\T\end{bmatrix}\T$
                    \State $\overline{\Y}_k \gets \begin{bmatrix}\overline{\Y}_{k,0}\T & \overline{\Y}_{k,1}\T & \cdots & \overline{\Y}_{k,C-1}\T\end{bmatrix}\T$

                    \State \LComment{Train the balanced classifier}
                    \State $\overline{\W}_k, \overline{\R}_k \gets $ \Call{Update}{$\Wh_{k}, \R_{k}, \overline{\X}_{k}, \overline{\Y}_{k}$}

                    \State \LComment{Use the balanced classifier for validation/inference}
                    \State \Call{Validate}{$\dataset{D}_{\text{val}}$, $\overline{\W}_k$}
                \EndProcedure
            \end{algorithmic}
        \end{algorithm}

    \subsection{Feature Extraction}
        \par Let $\dataset{D} = \{\dataset{D}_1, \dataset{D}_2, \dots, \dataset{D}_{K}\}$ of $C$ distinct classes be the overall training dataset with $K$ tasks that arrive phase by phase to train the model. For the dataset at the $k$-th task of size $N_k$, $\dataset{D}_k = \{(\Ximg_{k,1}, y_{k,1}), (\Ximg_{k,2}, y_{k,2}), \cdots, (\Ximg_{k,N_k}, y_{k,N_k})\}$ is the training set, where $\Ximg$ is an image tensor and $y$ is an integer ranging from $0$ to $C - 1$ that represents each distinct class.

        To utilize the power of pre-trained large models, we adopt a backbone network such as a ViT \cite{dosovitskiy2021image} to extract the features of images. Let
        \begin{equation}
            \f = f(\Ximg, \Wbackbone)
        \end{equation}
        be the features extracted by the backbone, where $\Wbackbone$ indicates the backbone weights. Then we use a linear layer of random weight $\Wrand$ followed by a ReLU activation inspired by various ACL methods \cite{zhuang2022acil, zhuang2023gkeal}, projecting the features into high dimension \cite{schmidt1992feed} as the input of the following classifier. The projected features $\bm{x}$ of shape $1 \times d$ can be defined as:
        \begin{equation}
            \Vec{x} = \ReLU\left(f(\Ximg, \Wbackbone)\Wrand\right).
        \end{equation}

    \subsection{Ridge Regression Classifier}
        \par To convert the classification problem into a ridge regression problem, we use the one-hot encoding to get target row vector $\Vec{y} = \operatorname{onehot}(y)$ of shape $1 \times C$. Thereby, we can represent each subset using two matrices $\dataset{D}_k \sim \{\X_k, \Y_k\}$ by stacking extracted feature vectors $\Vec{x}$ and target vectors $\Vec{y}$ vertically, where $\X_k \in \mathbb{R}^{N_k\times d}$ and $\Y_k \in \mathbb{R}^{N_k\times C}$.

        \par The training process of the ridge-regression classifier finds a weight matrix $\Wh_k \in \mathbb{R}^{d \times C}$ at the $k$-th task, linearly mapping the feature $\X_{1:k}$ to the label $\Y_{1:k}$
        \begin{equation}
            \Wh_k = \argmin{\W_k} \left(\SFN{\Y_{1:k} - \X_{1:k}\W_k} + \gamma \SFN{\W_k}\right),
        \end{equation}
        where $\gamma \ge 0$ is the coefficient of the regularization term and
        \begin{equation}
            \X_{1:k} = \begin{bmatrix}
                \X_{1} \\
                \X_{2} \\
                \vdots \\
                \X_{k} \\
            \end{bmatrix},\qquad
            \Y_{1:k} = \begin{bmatrix}
                \Y_{1} \\
                \Y_{2} \\
                \vdots \\
                \Y_{k} \\
            \end{bmatrix}.
        \end{equation}

        The optimal solution $\hat{\W_k} \in \mathbb{R}^{d \times C}$ is
        \begin{equation}\label{eq:W_i_optimial_soulution}
            \begin{split}
                \Wh_k &= (\X_{1:k}\T \X_{1:k} + \gamma \I)\inv \X_{1:k}\T \Y_{1:k} \\
                &= \left(\sum_{i=1}^{k}\X_{i}\T \X_{i} + \gamma \I\right)\inv \left(\sum_{i=1}^{k}\X_{i}\T \Y_{i}\right) \\
                &= \R_k\Q_k,
            \end{split}
        \end{equation}
        where $\R_k = (\sum_{i=1}^{K}\X_{i}\T \X_{i} + \gamma \I)\inv$ of shape $d \times d$ is a \textit{regularized feature autocorrelation matrix} and $\Q_k = \sum_{i=1}^{k}\X_{i}\T \Y_{i}$ of shape $d \times C$ is a \textit{cross correlation matrix}. $\R_k$ and $\Q_k$ capture the correlation information of $\X_{1:k}$ and $\Y_{1:k}$.

    \subsection{Continual Learning}
        Here, we give a recursive form of this analytical solution, which continually updates its weights online to obtain the same weights as training from scratch. This constructs a non-forgetting CL procedure.
        \begin{theorem}\label{thm:R_recursive}
            The calculation of the regularized feature autocorrelation matrix at task $k$, $\R_k = (\sum_{i=1}^{k}\X_{i}\T \X_{i} + \gamma \I)\inv$ is identical to its recursive form
            \begin{equation}
                \R_k = \R_{k - 1} - \R_{k - 1}\X_{k}\T(\I + \X_k\R_{k - 1}\X_{k}\T)\inv\X_{k}\R_{k - 1},
            \end{equation}
            where $\R_0 = \frac{1}{\gamma}\I$.
        \end{theorem}
        \begin{proof}
            According to the Woodbury matrix identity \cite{WoodburyIdentity_Woodbury1950}, for conformable matrices $\Mat{A}$, $\Mat{U}$, $\Mat{C}$, and $\Mat{V}$, we have
            \begin{equation}
                (\Mat{A} + \Mat{U}\Mat{C}\Mat{V})\inv = \Mat{A}\inv - \Mat{A}\inv\Mat{U}(\Mat{C}\inv + \Mat{V}\Mat{A}\inv\Mat{U})\inv\Mat{V}\Mat{A}\inv.
            \end{equation}
            Let $\Mat{A} = \R_{k}\inv$, $\Mat{U} = \X_{k}\T$, $\Mat{V} = \X_k$, and $\Mat{C} = \I$, we have
            \begin{equation}
                \begin{split}
                    \R_{k} & = (\R_{k-1}\inv + \X_k\T\X_k)\inv \\
                    & = \R_{k - 1} - \R_{k - 1}\X_{k}\T(\I + \X_k\R_{k - 1}\X_{k}\T)\inv\X_{k}\R_{k - 1},
                \end{split}
            \end{equation}
            which completes the proof.
        \end{proof}

        \begin{theorem}\label{thm:W_recursive}
            The weight of iterative classifier $\Wh_k$ obtained by \eqref{eq:W_i_optimial_soulution} is identical to its recursive form
            \begin{equation}
                \Wh_k = (\I - \R_{k}\X_{k}\T \X_{k})\Wh_{k-1} + \R_{k}\X_{k}\T \Y_k,
            \end{equation}
            where $\Wh_0 = \Mat{0}_{d\times C}$ is a zero matrix.
        \end{theorem}
        \begin{proof}
            \par According to
            \begin{equation}
                \Q_k = \sum_{i=1}^{k}\X_{i}\T \Y_{i} = \Q_{k - 1} + \X_{k}\T \Y_{k},
            \end{equation}
            \eqref{eq:W_i_optimial_soulution} can be derived to
            \begin{equation}\label{eq_appendix:W_k}
                \Wh_k = \R_k\Q_{k} = \R_k\Q_{k - 1} + \R_k \X_{k}\T \Y_{k}.
            \end{equation}
            \par According to Theorem \ref{thm:R_recursive},
            \begin{align}\label{eq_appendix:R_kQ_k-1}
            \R_k\Q_{k - 1}
            &= \R_{k - 1}\Q_{k - 1} - \R_{k - 1}\X_{k}\T\K_k\X_{k}\R_{k - 1}\Q_{k - 1} \nonumber\\
            &= (\I - \R_{k - 1}\X_{k}\T\K_k\X_{k})\Wh_{k - 1},
            \end{align}
            where $\K_k = (\I + \X_k\R_{k - 1}\X_{k}\T)\inv$ and $\K \in \mathbb{R}^{d \times d}$.
            \par Since
            \begin{equation}
                \K_k\K_k\inv = \K_k(\I + \X_k\R_{k - 1}\X_{k}\T) = \I,
            \end{equation}
            we have
            \begin{equation}
                \K_k = \I - \K_k\X_k\R_{k - 1}\X_{k}\T.
            \end{equation}
            \par Therefore,
            \begin{equation}
                \begin{split}
                &\R_{k - 1}\X_{k}\T\K_k = \R_{k - 1}\X_{k}\T(\I - \K_k\X_k\R_{k - 1}\X_{k}\T) \\
                &= (\R_{k - 1} - \R_{k - 1}\X_k\T\K_k\X_k\R_{k - 1})\X_{k}\T = \R_k\X_k\T,
                \end{split}
            \end{equation}
            which allows \eqref{eq_appendix:R_kQ_k-1} to be reduced to
            \begin{equation}\label{eq_appendix:R_kQ_K-1_recursive}
                \R_k\Q_{k - 1} = (\I - \R_{k}\X_{k}\T \X_{k})\Wh_{k}.
            \end{equation}
            \par Substituting \eqref{eq_appendix:R_kQ_K-1_recursive} into \eqref{eq_appendix:W_k} completes the proof.
        \end{proof}

        \par Notably, we calculate $\Wh_k$ using only data $\X_{k}$ and label $\Y_{k}$ at the $k$-th task, without involving any samples belonging to historical tasks like $\X_{k-1}$. Thus, our approach can be treated as an exemplar-free method. The pseudo-code of how it updates the weight of the classifier is listed in Algorithm \ref{algo:update}.

        \begin{algorithm}
            \caption{Update the weight of the classifier recursively}\label{algo:update}
            \begin{algorithmic}
                \Procedure{Update}{$\Wh_{k-1}$, $\R_{k-1}$, $\X_{k}$, $\Y_{k}$}
                    \State $\R_k \gets \R_{k - 1} - \R_{k - 1}\X_{k}\T(\I + \X_k\R_{k - 1}\X_{k}\T)\inv\X_{k}\R_{k - 1}$
                    \State $\Wh_k \gets (\I - \R_{k}\X_{k}\T \X_{k})\Wh_{k-1} + \R_{k}\X_{k}\T \Y_k$
                    \State \textbf{return} $\Wh_k$, $\R_k$
                \EndProcedure
            \end{algorithmic}
        \end{algorithm}

    \subsection{Pseudo-Features Generation}
        \par In the OCL process, the features of data extracted by backbone $\f$ come in a stream $\f_1, \f_2, \cdots, \f_n, \cdots$. We calculate the mean and variance of each different class. We can use the first $n$ samples of the same labels to evaluate the overall distribution of one object. We assume that the distribution of the features obtained by the backbone network follows the normal distribution and is pairwise independent.

        \par As data continue to arrive, our estimates of the feature distribution also evolve. Specifically, the mean and the variance can be updated recursively.

        The mean value of the features is calculated recursively by:
        \begin{equation}\label{eq:def_mean}
            \MU_n = \frac{1}{n}\sum_{i=1}^{n}\f_i = \frac{1}{n}\f_n + \frac{n - 1}{n}\MU_{n-1}.
        \end{equation}
        Similarly, there is also a recursive form of the square value:
        \begin{equation}
            \NU_n = \frac{1}{n}\sum_{i=1}^{n}\f_i^2 = \frac{1}{n}\f_n^2 + \frac{n - 1}{n}\NU_{n-1}.
        \end{equation}
        Using the mean value and the square value calculated recursively, we can get the estimation of feature variance:
        \begin{equation}
            \SIGMA^2_n = \frac{1}{n-1} \sum_{i=1}^{n}(\f_i - \MU_n)^2= \frac{n}{n - 1}(\NU_n - \MU_n^2).
        \end{equation}

        \par To address the issue of sample imbalance, we record the total count of samples from each category up to the current task. Subsequently, we offset the sample count of all categories to match that of the category with the most samples inspired by the oversampling methods \cite{SMOTE_Chawla_JAIR2002, LMLE_Huang_CVPR2016, OTOS_Yan_AAAI2019}. To do this, we recursively acquire the mean and variance of all current samples for each category and sample these compensatory samples randomly from the estimated normal distribution $\mathcal{N}(\MU_n, \SIGMA_n^2)$.

        \par For each different class, $\MU$ and $\SIGMA$ are usually different. Our method recursively calculates the values of $\MU$ and $\SIGMA$ for each class. We use $\MU^{(y)}$, $\NU^{(y)}$, and $\SIGMA^{(y)}$ to denote the mean, the mean square, and the standard deviation for the $y$-th class.

        \par These compensatory samples enter the same training process as if they were real samples, serving to update the classifier used solely for inference. Given the equivalence of our method for separate training and joint-learning, this process is equivalent to conducting complete analytical training for the full balanced data. Notably, the classifier in post-compensation learning is used only for the current task's inference, without influencing the $\R_{k}$ and $\Wh_{k}$ used in subsequent tasks.

    \subsection{Why AEF-OCL Overcomes Catastrophic Forgetting}
        \par For gradient-based methods, catastrophic forgetting can be attributed to the fundamental property named \textit{task-recency bias} \cite{LUCIR_Hou_CVPR2019} that predictions favor recently updated categories. This phenomenon is aggravated in driving scenarios with data imbalance, for example, when the data of new categories is much more than the data of old categories. To the authors' knowledge, no existing solutions exist for these gradient-based CL models to fully address catastrophic forgetting.

        \par As a branch of ACL, the AEF-OCL has the same \textit{absolute memorization property} \cite{zhuang2022acil} as other ACL methods. As indicated in Theorem \ref{thm:W_recursive}, the AEF-OCL recursively updates the weights of the classifier, which is identical to the weight directly learned on the joint dataset. This so-called \textit{weight-invariant property} gives AEF-OCL the same absolute memorization property as other ACL methods.

        \par Compared with other ACL methods, the AEF-OCL solves the data imbalance problem for the first time. Although the existing ACL methods solve catastrophic forgetting, their classifiers still suffer from data imbalance. The AEF-OCL eliminates the discrimination of the classifier caused by data imbalance, which makes it superior to other ACL methods in data imbalance scenarios such as autonomous driving.


\section{Experiments}\label{sec:experiments}
    In this section, we validate the proposed AEF-OCL by experimenting with it on the SODA10M \cite{han2021soda10m} dataset.
    \subsection{Introduction to the SODA10M Dataset}
        The SODA10M dataset is a large-scale self/semi-supervised object detection dataset for autonomous driving. It comprises 10 million unlabeled images and 20,000 labeled images with 6 representative object categories. The dataset's distribution is graphically represented in Fig. \ref{fig:counts} upon examination, showing that the dataset exhibits an imbalanced categorization. \textit{Car} constitutes a significant proportion, representing 55\% of the total dataset. Conversely, \textit{Tricycle} comprises a minuscule fraction, accounting for only 0.3\% of the overall data.
        \begin{figure}[!ht]
            \centering
            \includegraphics[width=\linewidth]{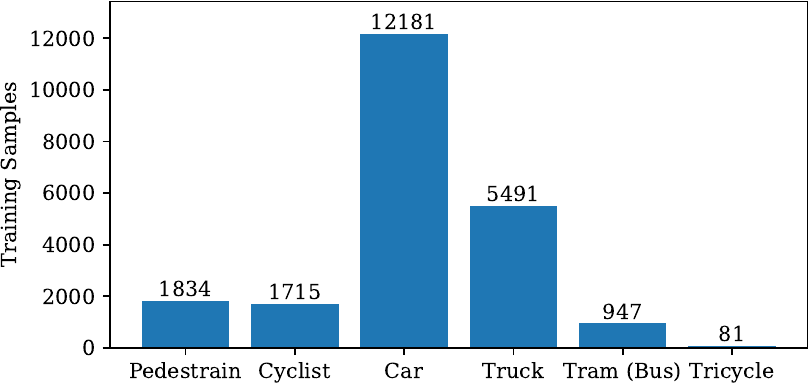}
            \caption{The number of training samples of each class.}
            \label{fig:counts}
            \vspace{-1em}
        \end{figure}

    \subsection{Evaluation Metric}
        \par Following the evaluation index proposed by the SODA10M paper \cite{han2021soda10m}, we use the \textit{average mean class accuracy} (AMCA) to evaluate our model. The AMCA is defined as:
        \begin{equation}
            AMCA = \frac{1}{T} \sum_{t} \frac{1}{C}\sum_{c}a_{c,t},
        \end{equation}
        where $a_{c,t}$ is the accuracy of class $c$ at task $t$.

        \par This metric is not affected by the number of samples in the training set. Categories with a few samples and those with numerous samples have equal weight in this metric. This indicator requires the model to have a considerable classification accuracy for both majority and minority classes. That is, the non-discrimination of the model.

    \subsection{Result Comparison}
        \par We perform our experiments on the SODA10M dataset. To utilize large models to obtain features that are easy to classify, we use the ViT-large/16 \cite{dosovitskiy2021image}, a ViT with $16 \times 16$ input patch size of 304.33M parameters and 61.55 GFLOPS, pre-trained on ImageNet-1k \cite{ImageNet_Deng_CVPR2009} provided by TorchVision \cite{torchvision2016} as a common backbone. For training details of the comparative methods, we use SGD for one epoch. We set the learning rate as 0.1 with a batch size of 10 and set both the momentum and the weight decay as 0. We use its generalized implementation of existing ACL methods introduced by \citet{GACL_Zhuang_NeurIPS2024}. For the ACIL, the DS-AL, and our AEF-OCL, we use the same random buffer of size 8192. For the replay-based methods, we set the memory size, the maximum number of images allowed to store, to 1000. Results are shown in TABLE \ref{tab:result_comparation}.

        \begin{table}[!ht]
            \centering
            \caption{The AMCA of ours and typical OCL methods}\label{tab:result_comparation}
            \begin{tblr}{X[c, m]Q[c, m, 0.25\linewidth]Q[c, m, 0.25\linewidth]}
                \toprule
                \textbf{Method}                    & \textbf{Memory Size} & \textbf{AMCA} (\%) \\ \midrule
                AGEM \cite{Chaudhry_AGEM_ICLR2019} & 1000                 & 41.61              \\
                EWC \cite{EWC2017nas}              & 0                    & 51.60              \\
                ACIL \cite{zhuang2022acil}         & 0                    & 55.01              \\
                DS-AL \cite{Zhuang_DSAL_AAAI2024}  & 0                    & 55.64              \\
                GKEAL \cite{zhuang2023gkeal}       & 0                    & 56.75              \\
                LwF \cite{li2018LWF}               & 0                    & 61.02              \\
                \SetRow{gray!10} \textbf{AEF-OCL}  & 0                    & \textbf{66.32}     \\ \bottomrule
            \end{tblr}
        \end{table}

        As indicated in TABLE \ref{tab:result_comparation}, among the exemplar-free methods, the AEF-OCL gives a superior performance (i.e., 66.32\% for AMCA). Other OCL techniques, such as the ACIL, perform less ideally (e.g., 55.01\%). There are two possible causes. First, methods such as the ACIL deal with incremental learning where data categories during training are mutually exclusive. On the SODA10M dataset, data categories usually appear jointly, allowing an easier CL operation. The other cause lies in the imbalance issue. This dataset is highly imbalanced, e.g., the \textit{Car}/\textit{Tricycle} categories have 55\%/0.3\% data distribution.

        The replay-based method AGEM exhibits comparatively lower precision (e.g., 41.61\%). This discrepancy could potentially be attributed to that the AGEM is based on a class-incremental paradigm. However, each training task in SODA10M could contain data of all categories, contradicting the AGEM training paradigm. Moreover, the imbalanced issue in OCL is also not properly treated in AGEM.

    \subsection{The Distribution of Features}
        The PFG module is set up on the assumption that the features obtained from the backbone roughly obey the normal distribution. To verify this, we use kernel density estimation \cite{KDE_Parzen_AMOS1956} to visualize the features. We can find from Fig. \ref{fig:distribution_class} that the features of different categories roughly follow a normal distribution with different means and variances.
        \begin{figure}[!ht]
            \centering
            \includegraphics[width=\linewidth]{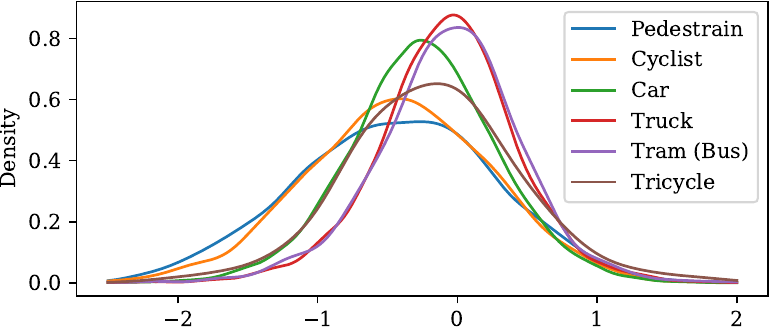}
            \caption{Distributions of the first element of features of different classes.}
            \label{fig:distribution_class}
        \end{figure}

        In addition, we plot the distribution of the features in a specific category (e.g., the \textit{Car} category in Fig. \ref{fig:distribution_feature}) and find that different feature elements of the same class also obey normal distribution with different means and variances, which verifies the assumption that the feature distribution is normal.

        \begin{figure}[!ht]
            \centering
            \includegraphics[width=\linewidth]{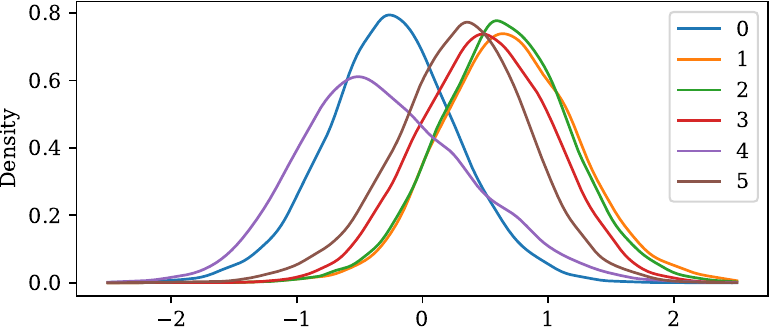}
            \caption{Distributions of the first 6 elements of features of the \textit{Car} class.}
            \label{fig:distribution_feature}
        \end{figure}

    \subsection{Why Not Update From Balanced Classifier}
        \par We use pseudo-samples (i.e., pseudo-features with their labels) to balance the weights of the classifier. During the online training, the previous pseudo-features of pseudo-samples may not accurately reflect the distribution of the overall data. Therefore, we retain the imbalanced iterative classifier, which is recursively trained on the features and labels from real data only. A balanced classifier is incrementally updated from the iterative classifier by the pseudo-samples for inference. In addition, this update strategy helps the AEF-OCL keep the same \textit{weight-invariant property} as the other ACL methods.

        The experiment in Fig. \ref{fig:strategies} shows that invariant to the value of the regularization term $\gamma$, updating from the iterative classifier has a higher AMCA than updating from the balanced classifier.
        \begin{figure}[!ht]
            \centering
            \includegraphics[width=\linewidth]{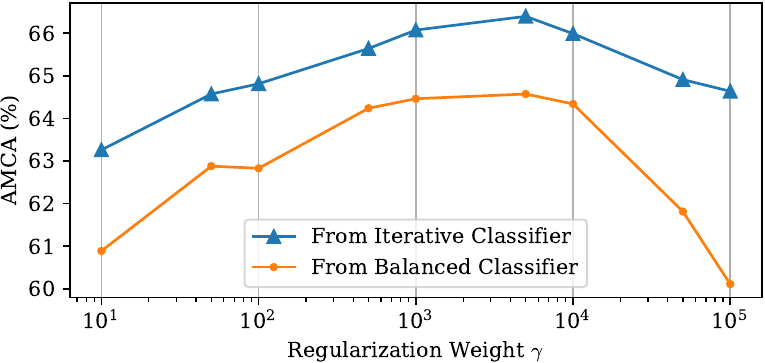}
            \caption{Different update strategies on different regularization weight.}
            \label{fig:strategies}
            \vspace{-1em}
        \end{figure}

    \subsection{Identical Distribution, Better Generator}
        \par It is important for the PFG module to generate pseudo-features with the same distribution as the real features. To show this, we introduce the noise coefficient $\alpha$, using $(\alpha\SIGMA)^2$ as the sampling variance, and study the impact of the PFG sampling strategy on the results. As shown in Fig. \ref{fig:noise}, when $\alpha$ is near 1, the AMCA is the highest, while other values encounter performance reduction. That is, when the estimation of $\SIGMA$ is correct, it benefits the algorithm. Otherwise, it will influence the performance to the extent proportional to the gap between the estimate and the ideal distribution.
        \begin{figure}[!ht]
            \centering
            \includegraphics[width=\linewidth]{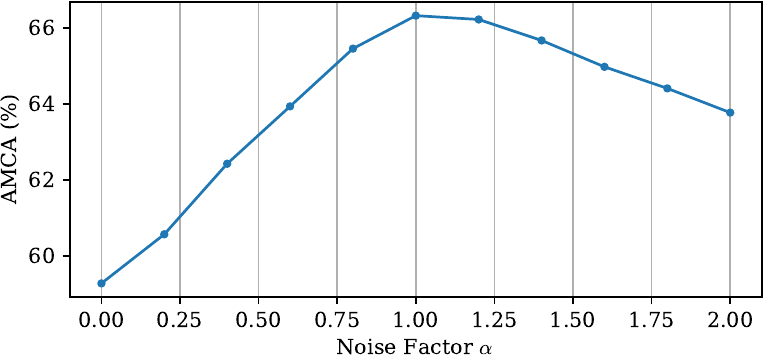}
            \caption{The AMCA on different noise factors.}
            \label{fig:noise}
            \vspace{-1em}
        \end{figure}

\section{Limitations and Future Works}
    \par The AEF-OCL needs a large-scale pre-trained backbone with powerful generalization ability. Online scenarios make it hard to adapt the backbone network to traffic datasets. This could motivate the exploration of adjustable backbones online.
    \par In addition, the high safety requirements of autonomous driving require us to explore security issues. Whether the AEF-OCL is robust enough to defend against attacks and whether the pseudo-features generated by the PFG module can enhance the robustness deserve further exploration.

\section{Conclusion}\label{sec:conclusions}
    \par In this paper, we have introduced the AEF-OCL, an OCL approach for imbalanced autonomous driving datasets based on a large-scale pre-trained backbone. Our method uses ridge regression as a classifier to solve the OCL problem in transportation by recursively calculating its analytical solution, establishing an equivalence between the CL and its joint-learning counterpart. Our AEF-OCL eliminates the need for historical samples, addresses privacy issues, and ensures data privacy. Furthermore, we have introduced the PFG module, which effectively combats data imbalance by generating pseudo-data through recursive distribution calculations on task-specific data. Experiments on the SODA10M dataset have validated the competitive performance of AEF-OCL in addressing OCL challenges associated with autonomous driving.

\section*{Acknowledgments}
This research was supported by
the Fundamental Research Funds for the Central Universities (2023ZYGXZR023, 2024ZYGXZR074),
the National Natural Science Foundation of China (62306117, 62406114, U23A20317),
the Guangzhou Basic and Applied Basic Research Foundation (2024A04J3681, 2023A04J1687),
the South China University of Technology-TCL Technology Innovation Fund,
the Guangdong Basic and Applied Basic Research Foundation (2024A1515010220), and
the CAAI-MindSpore Open Fund developed on Openl Community.

\printbibliography

\vspace{-3em}

\begin{IEEEbiography}
[{\includegraphics[width=1in,clip,keepaspectratio]{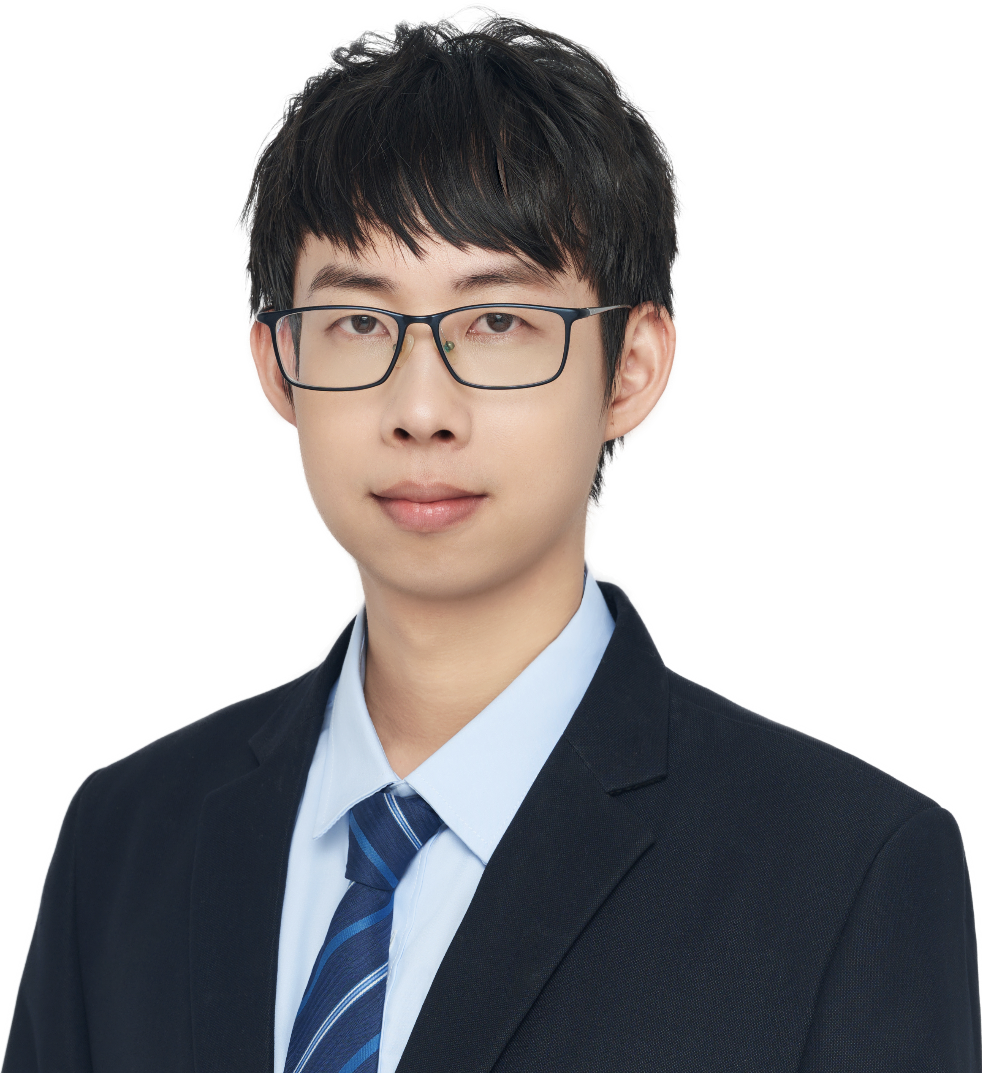}}]{Huiping Zhuang}
received B.S. and M.E. degrees from the South China University of Technology, Guangzhou, China, in 2014 and 2017, respectively, and the Ph.D. degree from the School of Electrical and Electronic Engineering, Nanyang Technological University, Singapore, in 2021.
\par He is currently an Associate Professor with the Shien-Ming Wu School of Intelligent Engineering, South China University of Technology. He has published more than 40 papers, including those in ICML, NeurIPS, CVPR, IEEE TNNLS, IEEE TSMC-S, and IEEE TGRS. He has served as a Guest Editor for Journal of Franklin Institute. His research interests include deep learning, AI computer architecture, and intelligent robots.
\end{IEEEbiography}

\begin{IEEEbiography}
[{\includegraphics[width=1in,clip,keepaspectratio]{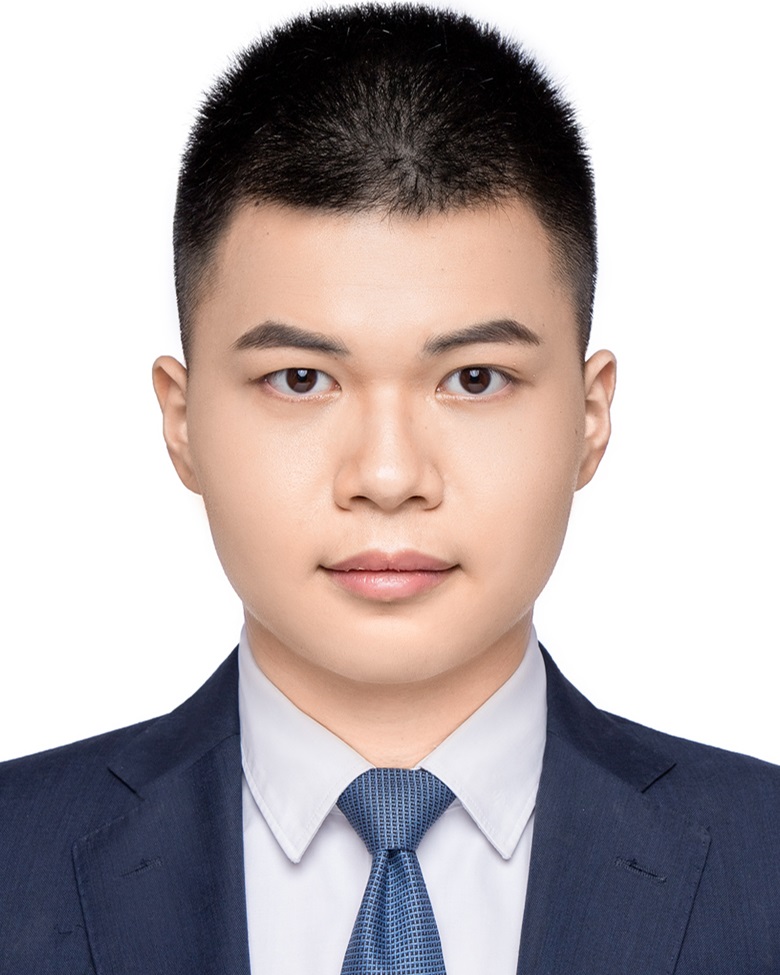}}]{Di Fang}
is an undergraduate student at the South China University of Technology. His research interests include machine learning and continual learning.
\end{IEEEbiography}

\vspace{-1.75em}

\begin{IEEEbiography}
[{\includegraphics[width=1in,clip,keepaspectratio]{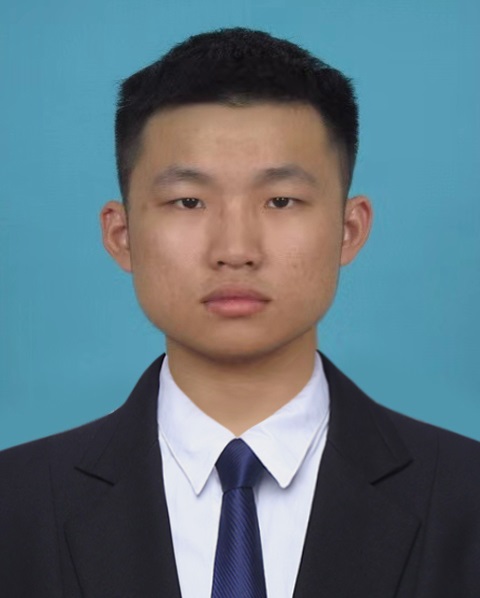}}]{Kai Tong}received the B.E. degree in the School of Automation, University of Electronic Science and Technology of China, and received the M.S. degree in University of Massachusetts Amherst.
\par He is currently studying for a Ph.D. degree in the Shien-Ming Wu School of Intelligent Engineering, South China University of Technology. His research interests include continual learning and large language models.
\end{IEEEbiography}

\vspace{-1.75em}

\begin{IEEEbiography}
[{\includegraphics[width=1in,clip,keepaspectratio]{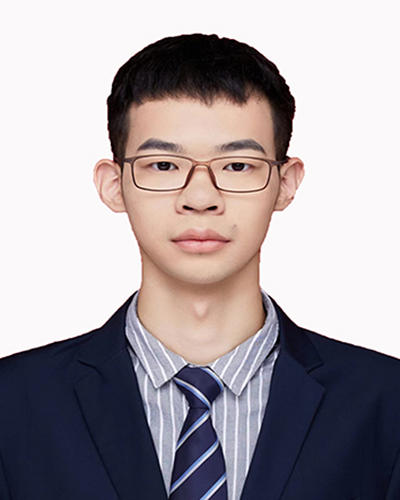}}]{Yuchen Liu}
received the B.E. degree in the Shien-Ming Wu School of Intelligent Engineering, South China University of Technology.
\par He is currently studying Master of Science program in the Department of Mechanical Engineering, The University of Hong Kong. His research interests include continual learning and deep learning.
\end{IEEEbiography}

\vspace{-1.75em}

\begin{IEEEbiography}
[{\includegraphics[width=1in,clip,keepaspectratio]{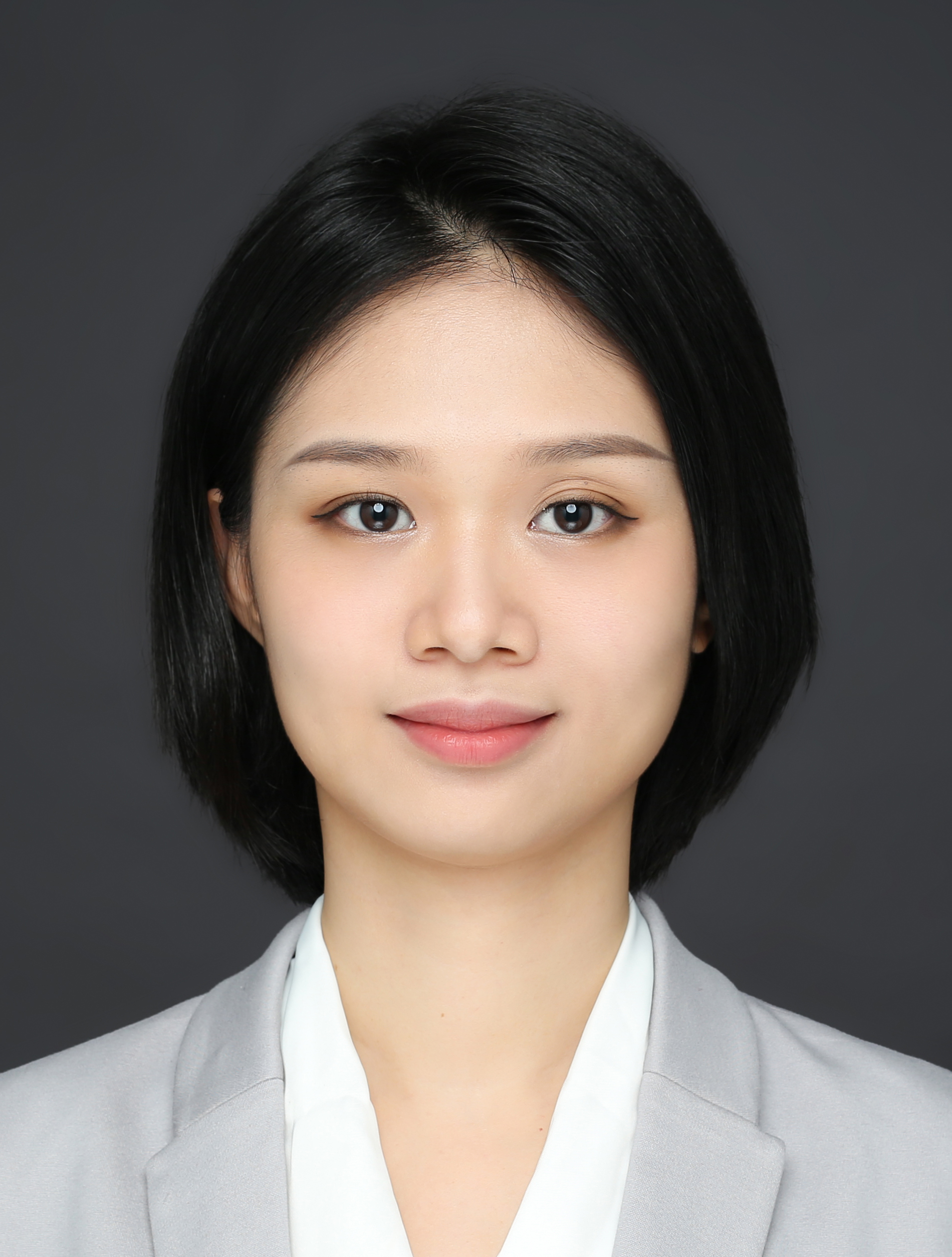}}]{Ziqian Zeng}
    obtained her Ph.D. degree in Computer Science and Engineering from The Hong Kong University of Science and Technology in 2021.
    \par She is currently an Associate Professor at the Shien-Ming Wu School of Intelligent Engineering, South China University of Technology. Her research interests include efficient inference, zero-shot learning, fairness, and privacy.
\end{IEEEbiography}

\vspace{-1.75em}

\begin{IEEEbiography}
[{\includegraphics[width=1in,clip,keepaspectratio]{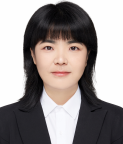}}]{Xu Zhou}
    is currently a professor with the Department of Information Science and Engineering, Hunan University, Changsha, China.
    \par She received the Ph.D. degree from the College of Computer Science and Electronic Engineering, Hunan University, in 2016. Her research interests include parallel computing, data management and spatial crowdsourcing.
\end{IEEEbiography}

\vspace{-1.75em}

\begin{IEEEbiography}
[{\includegraphics[width=1in,clip,keepaspectratio]{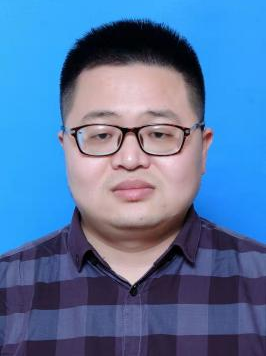}}]{Cen Chen}
received the Ph.D. degree in computer science from Hunan University, Changsha, China, in 2019. He previously worked as a Scientist with Institute of Infocomm Research (I2R), Agency for Science, Technology and Research (A*STAR), Singapore.
\par He currently works as a professor at the school of Future Technology of South China University of Technology and the Shenzhen Institute of Hunan University. His research interest includes parallel and distributed computing, machine learning and deep learning. He has published more than 60 articles in international conferences and journals on machine learning algorithms and parallel computing, such as HPCA, DAC, IEEE TC, IEEE TPDS, AAAI, ICDM, ICPP, and ICDCS. He has served as a Guest Editor for Pattern Recognition and Neurocomputing.
\end{IEEEbiography}

\end{document}